\documentclass{article}

\usepackage[final,nonatbib]{neurips_2019}

\usepackage[utf8]{inputenc} 
\usepackage[T1]{fontenc}    
\usepackage{hyperref}       
\usepackage{url}            
\usepackage{booktabs}       
\usepackage{amsfonts}       
\usepackage{nicefrac}       
\usepackage{microtype}      
\usepackage{hyperref}
\usepackage{url}
\usepackage{pifont}
\usepackage{color}
\usepackage{amsmath}
\usepackage{amssymb}
\usepackage{enumitem}
\usepackage{wrapfig,lipsum,booktabs}
\usepackage{tikz}
\usepackage{mathdots}
\usepackage{mathtools}
\usepackage{amsthm}
\usepackage{commath}
\usepackage{amsbsy}
\usepackage{cases}
\usepackage{cancel}
\usepackage{color}
\usepackage{array}
\usepackage{multirow}
\usepackage{gensymb}
\usepackage{tabularx}
\usepackage{booktabs}
\usetikzlibrary{fadings}
\usepackage{diagbox}
\usepackage{sidecap}
\usepackage{wrapfig}
\usepackage{enumitem}
\usepackage{enumitem}

\usepackage[linesnumbered,ruled,vlined]{algorithm2e}

\SetCommentSty{mycommfont}
\DeclareMathOperator{\Ker}{Ker}
\DeclareMathOperator{\sign}{sign}

\usepackage{amsmath,amsfonts,bm}

\def\1{\bm{1}}

\def\valpha{{\bm{\alpha}}}
\def\vdelta{{\bm{\delta}}}
\def\vDelta{{\bm{\Delta}}}
\def\va{{\bm{a}}}

\def\vg{{\bm{g}}}

\def\vs{{\bm{s}}}

\def\vv{{\bm{v}}}
\def\vw{{\bm{w}}}
\def\vx{{\bm{x}}}
\def\vy{{\bm{y}}}
\def\vz{{\bm{z}}}

\def\mA{{\bm{A}}}
\def\mB{{\bm{B}}}

\def\mD{{\bm{D}}}

\def\mH{{\bm{H}}}
\def\mI{{\bm{I}}}

\def\mK{{\bm{K}}}
\def\mL{{\bm{L}}}

\def\mP{{\bm{P}}}

\def\mS{{\bm{S}}}

\def\mU{{\bm{U}}}
\def\mV{{\bm{V}}}

\def\mX{{\bm{X}}}

\def\mSigma{{\bm{\Sigma}}}

\DeclareMathAlphabet{\mathsfit}{\encodingdefault}{\sfdefault}{m}{sl}
\SetMathAlphabet{\mathsfit}{bold}{\encodingdefault}{\sfdefault}{bx}{n}

\def\sI{{\mathbb{I}}}

\def\sR{{\mathbb{R}}}

\newtheorem{theorem}{Theorem}
\newtheorem{lemma}{Lemma}
\newtheorem{assumption}{Assumption}
\newtheorem{proposition}{Proposition}
\title{A Unified Framework for Data Poisoning Attack to Graph-based Semi-supervised Learning}

\author{%
  Xuanqing Liu \\
  Department of Computer Science\\
  UCLA\\
  \texttt{xqliu@cs.ucla.edu} \\
  \And
  Si Si \\
  Google Research\\
  \texttt{sisidaisy@google.com}\\
  \And
  Xiaojin Zhu\\
  Department of Computer Science\\
  University of Wisconsin-Madison\\
  \texttt{jerryzhu@cs.wisc.edu}\\
  \And
  Yang Li\\
  Google Research\\
  \texttt{liyang@google.com}\\
  \And
  Cho-Jui Hsieh\\
  Department of Computer Science\\
  UCLA\\
  \texttt{chohsieh@cs.ucla.edu} \\
}

\begin{document}

\maketitle

\begin{abstract}
In this paper, we proposed a general framework for data poisoning attacks to graph-based semi-supervised learning (G-SSL). In this framework, we first unify different tasks, goals and constraints into a single formula for data poisoning attack in G-SSL, then we propose two specialized algorithms to efficiently solve two important cases --- poisoning regression tasks under $\ell_2$-norm constraint and classification tasks under $\ell_0$-norm constraint. In the former case, we transform it into a non-convex trust region problem and show that our gradient-based algorithm with delicate initialization and update scheme finds the (globally) optimal perturbation. For the latter case, although it is an NP-hard integer programming problem, we propose a probabilistic solver that works much better than the classical greedy method. Lastly, we test our framework on real datasets and evaluate the robustness of G-SSL algorithms. For instance, on the MNIST binary classification problem (50000 training data with 50 labeled), flipping two labeled data is enough to make the model perform like random guess (around 50\% error).  
\end{abstract}
\section{Introduction}
Driven by the hardness of labeling work, graph-based semi-supervised learning (G-SSL)~\cite{zhu2002learning,zhu2003semi,chapelle2009semi} has been widely used to boost the quality of models using easily accessible unlabeled data. The core idea behind it is that both labeled and unlabeled data coexist in the same manifold. For instance, in the transductive setting, we have label propagation~\cite{zhu2002learning} that transfers the label information from labeled nodes to neighboring nodes according to their proximity. While in the inductive case, a graph-based manifold regularizer can be added to many existing supervised learning models to enforce the smoothness of predictions on the data manifold~\cite{belkin2006manifold,sindhwani2005linear}.  G-SSL has received a lot of attention; many of the applications are safety-critical such as drug discovery~\cite{zhang2015label} and social media mining~\cite{speriosu2011twitter}.
\par
We aim to develop systematic and efficient data poisoning methods for poisoning G-SSL models. Our idea is partially motivated by the recent researches on the robustness of machine learning models to adversarial examples~\cite{goodfellow2015explaining,szegedy2013intriguing}. These works mostly show that carefully designed, slightly perturbed inputs -- also known as adversarial examples -- can substantially degrade the performance of many machine learning models. We would like to tell apart this problem from our setting: adversarial attacks are performed during the testing phase and applied to test data~\cite{carlini2017towards,chen2018ead,athalye2018obfuscated,cheng2018seq2sick,papernot2016crafting,cheng2019query}, whereas data poisoning attack is conducted during training phase~\cite{mei2015using,koh2017understanding,xiao2015support,li2016data,zhao2018data}, and perturbations are added to training data only. In other words, data poisoning attack concerns about \emph{how to imperceptibly change the training data to affect testing performance.} As we can imagine, this setting is more challenging than testing time adversarial attacks due to the hardness of propagating information through a sophisticated training algorithm.
\par
Despite the efforts made on studying poisoning attack to supervised models~\cite{mei2015using,koh2017understanding,xiao2015support,li2016data,zhao2018data}, the robustness of semi-supervised algorithms has seldom been studied 
and many related questions remain unsolved. 
For instance, are semi-supervised learning algorithms sensitive to small perturbation of labels? And how do we formally measure the robustness of these algorithms? 
\par
In this paper, we initiate the first systematic study of data poisoning attacks against G-SSL. We mainly cover the widely used label propagation algorithm, but similar ideas can be applied to poisoning manifold regularization based SSL as well (see Appendix 4.2). To poison semi-supervised learning algorithms, we can either change the training labels or features. For label poisoning, we show it is a constrained quadratic minimization problem, and depending on whether it is a regression or classification task, we can take a continuous or discrete optimization method. For feature poisoning, we conduct gradient-based optimization with group Lasso regularization to enforce group sparsity (shown in Appendix 4.2). Using the proposed algorithms, we answer the questions mentioned above with several experiments. Our contributions can be summarized as follows: 
\begin{itemize}[noitemsep,leftmargin=*]
    \item We propose a framework for data poisoning attack to G-SSL that 1) includes both classification and regression cases, 2) works under various kinds of constraints, and 3) assumes both complete and incomplete knowledge of algorithm user (also called ``victim'').
    \item For label poisoning to regression task, which is a nonconvex trust region problem, we design a specialized solver that can find a global minimum in asymptotically linear time. 
    \item For label poisoning attack to classification task, which is an NP-hard integer programming problem, we propose a novel probabilistic solver that works in combination with gradient descent optimizer. Empirical results show that our method works much better than classical greedy methods.
    \item We design comprehensive experiments using the proposed poisoning algorithms on a variety of problems and datasets. 
\end{itemize}
In what follows, we refer to the party running poisoning algorithm as the attacker, and the party doing the learning and inference work as the victim.
\section{Related Work}
Adversarial attacks have been extensively studied recently. Many recent works consider the test time attack, where the model is fixed, and the attacker slightly perturbs a testing example to change the model output completely~\cite{szegedy2013intriguing}. We often formulate the attacking process as an optimization problem~\cite{carlini2017towards}, which can be solved in the white-box setting. 
In this paper, we consider a different area called \emph{data poisoning attack}, where we run the attack during  training time --- an attacker can carefully modify (or add/remove) data in the training set so that the model trained on the poisoned data 
either has significantly degraded performance~\cite{xiao2015support,mei2015using} or has some desired properties~\cite{chen2017targeted,li2016data}.
As we mentioned, this is usually harder than test time attacks since the model is not predetermined. 
Poisoning attacks have been studied in several applications, including multi-task learning~\cite{zhao2018data}, image classification~\cite{chen2017targeted}, matrix factorization for recommendation systems~\cite{li2016data} and online learning~\cite{wang2018data}. However, they did not include semi-supervised learning, and the resulting algorithms are quite different from us.
\par
To the best of our knowledge,
\cite{dai2018adversarial,zugner2018adversarial,wang2018attack} are the only related works on attacking semi-supervised learning models. They conduct \textbf{test time attacks} to Graph Convolutional Network (GCN). In summary, their contributions are different from us in several aspects:
1) the GCN algorithm is quite different from the classical SSL algorithms considered in this paper (e.g. label propagation and manifold regularization). Notably, we only use feature vectors and the graph will be constructed manually with kernel function. 2) Their works are restricted to testing time attacks by assuming the model is \textbf{learned and fixed}, and the goal of attacker is to find a perturbation to fool the established model. Although there are some experiments in \cite{zugner2018adversarial} on poisoning attacks, the perturbation is still generated from \textbf{test time attack} and they did not design task-specific algorithms for the poisoning in the \textbf{training time}. 
In contrast, we consider the data poisoning problem, which happens before the victim trains a model.
\section{Data Poisoning Attack to G-SSL}
\subsection{Problem setting}
We consider the graph-based semi-supervised learning (G-SSL) problem. The input include labeled data $\mX_l\in\sR^{n_l\times d}$ and unlabeled data $\mX_u\in\sR^{n_u\times d}$, we define the whole features $\mX=[\mX_l;\mX_u]$. Denoting the labels of $\mX_l$ as $\vy_l$, our goal is to predict the labels of test data $\vy_u$. The learner applies algorithm $\mathcal{A}$ to predict $\vy_u$ from available data $\{\mX_l, \vy_l, \mX_u\}$. Here we restrict $\mathcal{A}$ to label propagation method, where we first generate a graph with adjacency matrix $\mS$ from Gaussian kernel: $\mS_{ij}=\exp(-\gamma\|\vx_i-\vx_j\|^2)$, where the subscripts $\vx_{i(j)}$ represents the $i(j)$-th row of $\mX$. Then the graph Laplacian is calculated by $\mL=\mD-\mS$, where $\mD=\text{diag}\{\sum_{k=1}^n\mS_{ik}\}$ is the degree matrix. The unlabeled data is then predicted through energy minimization principle~\cite{zhu2003semi}
\begin{equation}
\label{eq:label-prop-problem}
    \min_{\hat{\vy}}\  \frac{1}{2}\sum_{i,j}\mS_{ij}(\hat{\vy}_i-\hat{\vy}_j)^2=\hat{\vy}^\intercal \mL\hat{\vy}, \quad\text{s.t. }\quad\hat{\vy}_{:l}=\vy_l.
\end{equation}
The problem has a simple closed form solution $\hat{\vy}_u=(\mD_{uu}-\mS_{uu})^{-1}\mS_{ul}\vy_l$, where we define $\mD_{uu}=\mD_{[0:u,0:u]}$, $\mS_{uu}=\mS_{[0:u,0:u]}$ and $\mS_{ul}=\mS_{[0:u,0:l]}$. Now we consider the attacker who wants to greatly change the prediction result $\vy_u$ by perturbing the \emph{training data} $\{\mX_l, \vy_l\}$ by small amounts $\{\Delta_x, \vdelta_y\}$ respectively, where $\Delta_x\in\sR^{n_l\times d}$ is the perturbation matrix , and $\vdelta_y\in\sR^{n_l}$ is a vector. This seems to be a simple problem at the first glance, however, 
we will show that the problem of finding optimal perturbation is often intractable, and therefore provable and effective algorithms are needed. To sum up, the problem have several degrees of freedom:
\begin{itemize}[noitemsep,leftmargin=*]
    \item \textbf{Learning algorithm}: Among all graph-based semi-supervised learning algorithms, we primarily focus on the label propagation method; however, we also discuss manifold regularization method in Appendix 4.2.
    \item \textbf{Task}: We should treat the regression task and classification task differently because the former is inherently a continuous optimization problem while the latter can be transformed into integer programming.
    \item \textbf{Knowledge of attacker}: Ideally, the attacker knows every aspect of the victim, including training data, testing data, and training algorithms. However, we will also discuss incomplete knowledge scenario; for example, the attacker may not know the exact value of hyper-parameters. 
    \item \textbf{What to perturb}: We assume the attacker can perturb the label or the feature, but not both. We made this assumption to simplify our discussion and should not affect our findings. 
    \item \textbf{Constraints}: We also assume the attacker has limited capability, so that (s)he can only make small perturbations. It could be measured $\ell_2$-norm or sparsity.
\end{itemize}
\subsection{Toy Example}
\begin{wrapfigure}{l}{0.4\textwidth}
\centering
\resizebox{0.98\linewidth}{!}{
\includegraphics{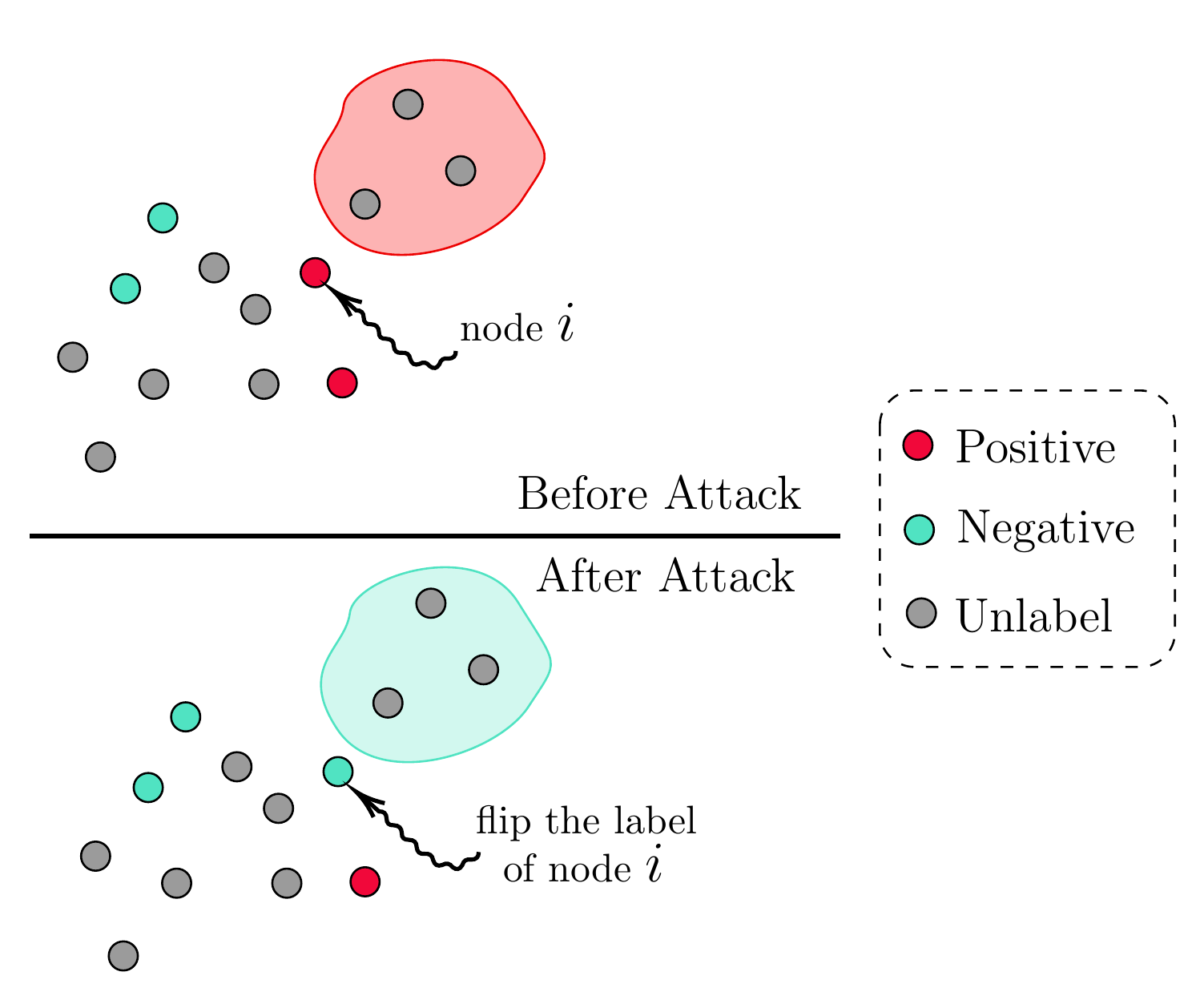}
}
\caption{We show a toy example that illustrates the main idea of the poisoning attack against SSL. By flipping just one training data from positive to negative, the prediction of the whole shaded area will be changed.}\label{fig:motivating-ex}
\vspace{-20pt}
\end{wrapfigure}
We show a toy example in Figure~\ref{fig:motivating-ex} to motivate the data poisoning attack to graph semi-supervised learning (let us focus on label propagation in this toy example). In this example, the shaded region is very close to node-$i$ and yet quite far from other labeled nodes. After running label propagation, all nodes inside the shaded area will be predicted to be the same label as node-$i$. That gives the attacker a chance to manipulate the decision of all unlabeled nodes in the shaded area at the cost of flipping just one node. For example, in Figure~\ref{fig:motivating-ex}, if we change node-$i$'s label from positive to negative, the predictions in the shaded area containing three nodes will also change from positive to negative.
\par
Besides changing the labels, another way to attack is to perturb the features $\mX$ so that the graph structure $\mS$ changes subtly (recall the graph structure is constructed based on pair-wise distances). For instance, we can change the features so that node $i$ is moved away from the shaded region, while more negative label points are moved towards the shaded area. Then with label propagation, the labels of the shaded region will be changed from positive to negative as well. We will examine both cases in the following sections.
\subsection{A unified framework}
The goal of poisoning attack is to modify the data points to maximize the error rate (for classification) or RMSE score (for regression); thus we write the objective as
\begin{equation}
    \label{eq:unified-obj}
    \min_{\substack{\vdelta_y\in\mathcal{R}_1\\\Delta_x\in\mathcal{R}_2}} -\frac{1}{2}\Big\|g\Big((\mD'_{uu}-\mS'_{uu})^{-1}\mS'_{ul}(\vy_l+\vdelta_y)\Big) - h(\vy_u)\Big\|^2_2 \quad\text{s.t. } \{\mD', \mS'\}=\Ker_{\gamma}(\mX_l+\Delta_x).
\end{equation}
To see the flexibility of Eq.~\eqref{eq:unified-obj} in modeling different tasks, different knowledge levels of attackers or different budgets, we decompose it into following parts that are changeable in real applications:
\begin{itemize}[noitemsep,leftmargin=*]
    \item $\mathcal{R}_1$/$\mathcal{R}_2$ are the constraints on $\vdelta_y$ and $\Delta_x$. For example, $\mathcal{R}_1=\{\|\vdelta_y\|_2\le d_{\max}\}$ restricts the perturbation $\vdelta_y$ to be no larger than $d_{\max}$; while $\mathcal{R}_1=\{\|\vdelta_y\|_0\le c_{\max}\}$ makes the solution to have at most $c_{\max}$ non-zeros. As to the choices of $\mathcal{R}_2$, besides $\ell_2$ regularization, we can also enforce group sparsity structure, where each row of $\Delta_x$ could be all zeros.
    \item $g(\cdot)$ is the task dependent squeeze function, for classification task we set $g(x)=\sign(x)$ since the labels are discrete and we evaluate the accuracy; for regression task it is identity function $g(x)=x$, and $\ell_2$-loss is used.
    \item $h(\cdot)$ controls the knowledge of unlabeled data. If the adversary knows the ground truth very well, then we simply put $h(\vy_u)=\vy_u$; otherwise one has to estimate it from Eq.~\eqref{eq:label-prop-problem}, in other words, $h(\vy_u)=\hat{\vy}_u=g\big((\mD_{uu}-\mS_{uu})^{-1}\mS_{ul}\vy_l\big)$.
    \item $\Ker_{\gamma}$ is the kernel function parameterized by $\gamma$, we choose Gaussian kernel throughout.
    \item Similar to $\mS$, the new similarity matrix $\mS'$ is generated by Gaussian kernel with parameter $\gamma$, except that it is now calculated upon poisoned data $\mX_l+\Delta_x$.
    \item Although not included in this paper, we can also formulate targeted poisoning attack problem by changing min to max and let $h(\vy_u)$ be the target. 
\end{itemize}
There are two obstacles to solving Eq.~\ref{eq:unified-obj}, that make our algorithms non-trivial. First, the problem is naturally \emph{non-convex}, making it hard to determine whether a specific solution is globally optimal; secondly, in classification tasks where our goal is to maximize the testing time error rate, the objective is \emph{non-differentiable} under \emph{discrete} domain. Besides, even with hundreds of labeled data, the domain space can be unbearably big for brute force search and yet the greedy search is too myopic to find a good solution (as we will see in experiments).
\par
In the next parts, we show how to tackle these two problems separately. Specifically, in the first part, we propose an efficient solver designed for data poisoning attack to the regression problem under various constraints. Then we proceed to solve the discrete, non-differentiable poisoning attack to the classification problem.
\subsection{\label{sec:regression}Regression task, (un)known label}
\begin{algorithm}[h]
\caption{\label{alg:tr_solver}Trust region problem solver}
\SetAlgoLined
\KwData{Vector $\vg$, symmetric indefinite matrix $\mH$ for problem $\min_{\|\vz\|\le 1}\frac{1}{2}\vz^\intercal\mH\vz +\vg^\intercal\vz$.}
\KwResult{Approximate solution $\vz^*$.}
Initialize $\vz_0=-0.5\frac{\vg}{\|\vg\|}$ and step size $\eta$\;
 \tcc{Phase I: iterate inside sphere $\|\vz_t\|<1$}
 \While{ $\|\vz_t\|<1$ }{
    $\vz_{t+1}=\vz_t-\eta(\mH \vz_t+\vg)$\;
 }
 \tcc{Phase II: iterate on the sphere $\|\vz_t\|=1$}
$\vz_{t'} = \vz_t$\;
\While{$t<\mathrm{max\_iter}$}{
    Choose $\alpha_{t'}$ by line search and do the following projected gradient descent on sphere\;
    $\vz_{t'+1}=\frac{\vz_{t'}-\alpha_{t'}(\mI_d-\vz_{t'}\vz_{t'}^\intercal)(\mH\vz_{t'}+\vg)}{\|\vz_{t'}-\alpha_{t'}(\mI_d-\vz_{t'}\vz_{t'}^\intercal)(\mH\vz_{t'}+\vg)\|}$\;
}
{\bfseries Return} $\vz_{\text{max\_iter}}$
\end{algorithm}
We first consider the regression task where only label poisoning is allowed. This simplifies Eq.~\eqref{eq:unified-obj} as
\begin{subnumcases}{\min_{\|\vdelta_y\|_2\le d_{\max}}}
    -\frac{1}{2}\Big\| (\mD_{uu}-\mS_{uu})^{-1}\mS_{ul}\vdelta_y \Big\|^2_2 & \text{(estimated label)}\label{eq:estimate-label}\\
    -\frac{1}{2}\Big\|(\mD_{uu}-\mS_{uu})^{-1}\mS_{ul}(\vy_l+\vdelta_y) - \vy_u\Big\|^2_2 & \text{(true label)}\label{eq:true-label}
\end{subnumcases}
Here we used the fact that $\hat{\vy}_u=\mK\vy_l$, where we define $\mK=(\mD_{uu}-\mS_{uu})^{-1}\mS_{ul}$. We can solve Eq.~\eqref{eq:estimate-label} by SVD; it's easy to see that the optimal solution should be $\vdelta_y=\pm d_{\max} \vv_1$ and $\vv_1$ is the top right sigular vector if we decompose $(\mD_{uu}-\mS_{uu})^{-1}\mS_{ul}=\mU\mSigma\mV^\intercal$. However, \eqref{eq:true-label} is less straightforward, in fact it is a non-convex trust region problem, which can be generally formulated as
\begin{equation}
    \min_{\|\vz\|_2\le d_{\max}}f(\vz)=\frac{1}{2}\vz^\intercal\mH\vz+\vg^\intercal\vz,\quad \mH \text{ is indefinite}.
\end{equation}
Our case~\eqref{eq:true-label} can thus be described as $\mH=-\mK^\intercal\mK\preceq \bm{0}$ and $\vg=\mK^\intercal(\vy_u-\hat{\vy}_u)$.
Recently \cite{hazan2016linear} proposed a sublinear time solver that is able to find a global minimum in $\mathcal{O}(M/\sqrt{\epsilon})$ time. Here we propose an asymptotic linear algorithm based purely on gradient information, which is stated in Algorithm~\ref{alg:tr_solver} and Theorem~\ref{th:linear-converge}. In Algorithm~\ref{alg:tr_solver} there are two phases, in the following theorems, we show that the phase I ends within finite iterations, and phase II converges with an asymptotic linear rate. We postpone the proof to Appendix 1. 
\begin{theorem}[Convergent]\label{th:converge}
Suppose the operator norm $\|\mH\|_{\mathrm{op}}=\beta$, by choosing a step size $\eta<1/\beta$ with initialization $\vz_0=-\alpha\frac{\vg}{\|\vg\|}$, $0<\alpha<\min(1, \frac{\|\vg\|^3}{|\vg^\intercal\mH\vg|})$. Then iterates $\{\vz_t\}$ generated from Algorithm~\ref{alg:tr_solver} converge to the global minimum.
\end{theorem}
\begin{lemma}[Finite phase I]\label{le:bound_t}
Since $\mH$ is indefinite, $\lambda_1=\lambda_{\min}(\mH)<0$, and $\vv_1$ is the corresponding eigenvector. Denote $\va^{(1)}=\va^\intercal \vv_1$ is the projection of any $\va$ onto $\vv_1$, let $T_1$ be number of iterations in phase I of Algorithm~\ref{alg:tr_solver}, then:
\begin{equation}\label{eq:bound-t}
T_1\le \log(1-\eta\lambda_1)^{-1}\Big[\log\big(\frac{1}{\eta|\vg^{(1)}|}-\frac{1}{\eta\lambda_1}\big)
-\log\big(\frac{-\vz_0^{(1)}}{\eta \vg^{(1)}}-\frac{1}{\eta\lambda_1}\big)\Big].
\end{equation}
\end{lemma}
\begin{theorem}[Asymptotic linear rate]\label{th:linear-converge}
Let $\{\vz_t\}$ be an infinite sequence of iterates generated by Algorithm~\ref{alg:tr_solver}, suppose it converges to $\vz^*$ (guaranteed by Theorem~\ref{th:converge}), let $\lambda_{\mH,\min}$ and $\lambda_{\mH,\max}$ be the smallest and largest eigenvalues of $\mH$. Assume that $\vz^*$ is a local minimizer then $\lambda_{\mH,\min}>0$ and given $r$ in the interval $(r_*, 1)$ with $r_*=1-\min\big(2\sigma\bar{\alpha}\lambda_{\mH,\min}, 4\sigma(1-\sigma)\beta\frac{\lambda_{\mH,\min}}{\lambda_{\mH,\max}}\big)$, $\bar{\alpha}$, $\sigma$ are line search parameters. There exists an integer $K$ such that:
$$
f(\vz_{t+1})-f(\vz^*)\le r\big(f(\vz_t)-f(\vz^*)\big)
$$
for all $t\ge K$.
\end{theorem}
\subsection{\label{sec:classification}Classification task}
As we have mentioned, data poisoning attack to classification problem is more challenging, as we can only \emph{flip} an unnoticeable fraction of training labels. This is inherently a combinatorial optimization problem. For simplicity, we restrict the scope to binary classification so that $\vy_l\in\{-1, +1\}^{n_l}$, 
and the labels are perturbed as $\tilde{\vy}_l=\vy_l\odot\vdelta_y$, where $\odot$ denotes Hadamard product and $\vdelta_y=[\pm 1, \pm 1, \dots, \pm 1]$. For restricting the amount of perturbation, we replace the norm constraint in Eq.~\eqref{eq:estimate-label} with integer constraint $\sum_{i=1}^{n_l}\mathbb{I}_{\{\vdelta_y[i]=-1\}}\le c_{\max}$, where $c_{\max}$ is a user pre-defined constant. In summary, the final objective function has the following form
\begin{equation}
    \label{eq:obj-func-perturb_y-discrete}
        \min_{\vdelta_y\in\{+1,-1\}^{n_l}}\  -\frac{1}{2}\Big\|g\big(\mK(\vy_l\odot\vdelta_y)\big)-(\vy_u\text{ or }\hat{\vy}_u)\Big\|^2,\quad 
        \text{s.t.}\quad\sum\nolimits_{i=1}^{n_l}\mathbb{I}_{\{\vdelta_y[i]=-1\}}\le c_{\max},
\end{equation}
where we define $\mK=(\mD_{uu}-\mS_{uu})^{-1}\mS_{ul}$ and $g(x)=\sign(x)$, so the objective function directly relates to error rate. Notice that the feasible set contains around $\sum_{k=0}^{c_{\max}}{n_l\choose k}$ solutions, making it almost impossible to do an exhaustive search. A simple alternative is greedy search: first initialize $\vdelta_y=[+1,+1,\dots,+1]$, then at each time we select index $i\in[n_l]$ and try flip $\vdelta_y[i]= +1\to -1$, such that the objective function~\eqref{eq:obj-func-perturb_y-discrete} decreases the most. Next, we set $\vdelta_y[i]=-1$. We repeat this process multiple times until the constraint in~\eqref{eq:obj-func-perturb_y-discrete} is met.

Doubtlessly, the greedy solver is myopic. The main reason is that the greedy method cannot explore other flipping actions that appear to be sub-optimal within the current context, despite that some sub-optimal actions might be better in the long run. Inspired by the bandit model, we can imagine this problem as a multi-arm bandit, with $n_l$ arms in total. And we apply a strategy similar to $\epsilon$-greedy: each time we assign a high probability to the best action but still leave non-zero probabilities to other ``actions''. The new strategy can be called \emph{probabilistic method}, specifically, we model each action $\vdelta_y=\pm 1$ as a Bernoulli distribution, the probability of ``flipping'' is $P[\vdelta_y=-1]=\valpha$. The new loss function is just an expectation over Bernoulli variables
\begin{equation}\label{eq:obj-func-randomness}
                \min_{\valpha}\  \bigg\{ \mathcal{L}(\valpha):=
            -\frac{1}{2}\mathop{\mathbb{E}}_{\vz\sim \mathcal{B}(\bm{1}, \valpha)}
            \left[
            \big\|g\big(\mK(\vy_l\odot\vz)\big)-(\vy_u\text{ or }\hat{\vy}_u)\big\|^2\right]
            +\frac{\lambda}{2}\cdot \|\valpha\|_2^2 
                \bigg\}.
\end{equation}
Here we replace the integer constraint in Eq.~\ref{eq:obj-func-perturb_y-discrete} with a regularizer $\frac{\lambda}{2}\|\valpha\|_2^2$, the original constraint is reached by selecting a proper $\lambda$. Once  problem~\eqref{eq:obj-func-randomness} is solved, we craft the actual perturbation $\vdelta_y$ by setting $\vdelta_y[i]=-1$ if $\valpha[i]$ is among the top-$c_{\max}$ largest elements.
\par
To solve Eq.~\eqref{eq:obj-func-randomness}, we need to find a good gradient estimator. Before that, we replace $g(x)=\sign(x)$ with $\tanh(x)$ to get a continuously differentiable objective. We borrow the idea of ``reparameterization trick''~\cite{figurnov2018implicit,tucker2017rebar} to approximate $\mathcal{B}(\bm{1}, \valpha)$ by a continuous random vector
\begin{equation}\label{eq:Gumbel-reparameterize}
    \vz\triangleq\vz(\valpha, \Delta_G)=\frac{2}{1+\exp\Big(\frac{1}{\tau}\big(\log\frac{\valpha}{1-\valpha}+\Delta_G\big)\Big)}-1
        \in (-1, 1),
\end{equation}
where $\Delta_G\sim\vg_1-\vg_2$ and $\vg_{1,2}\overset{\text{iid}}{\sim}\text{Gumbel}(0, 1)$ are two Gumbel distributions. $\tau$ is the temperature controlling the steepness of sigmoid function: as $\tau\to 0$, the sigmoid function point-wise converges to a stair function. Plugging \eqref{eq:Gumbel-reparameterize} into \eqref{eq:obj-func-randomness}, the new loss function becomes
\begin{equation}
        \mathcal{L}(\valpha):=-\frac{1}{2}\mathop{\mathbb{E}}_{\Delta_G} \left[\big\|g\big(\mK(\vy_l\odot\vz(\valpha,\Delta_G))\big)-(\vy_u\text{ or }\hat{\vy}_u)\big\|^2\right]
        +\frac{\lambda}{2}\cdot \|\valpha\|_2^2.
\end{equation}
Therefore, we can easily obtain an unbiased, low variance gradient estimator via Monte Carlo sampling from $\Delta_G=\vg_1-\vg_2$, specifically
\begin{equation}
    \label{eq:stochastic-gradient}
    \frac{\partial\mathcal{L}(\valpha)}{\partial \valpha}\approx -\frac{1}{2}\frac{\partial}{\partial\valpha}\big\|g\big(\mK(\vy_l\odot\vz(\valpha,\Delta_G))\big)-(\vy_u\text{ or }\hat{\vy}_u)\big\|^2+\lambda\valpha.
\end{equation}
Based on that, we can apply many stochastic optimization methods, including SGD and Adam~\cite{kingma2014adam}, to finalize the process. In the experimental section, we will compare the greedy search with our probabilistic approach on real data.

\section{Experiments}
In this section, we will show the effectiveness of our proposed data poisoning attack algorithms for regression and classification tasks on graph-based SSL.
\begin{wraptable}{r}{0.5\linewidth}
\label{tab:data}
\vspace{-10pt}
\caption{Dataset statistics. Here $n$ is the total number of samples, $d$ is the dimension of feature vector and $\gamma^*$ is the optimal $\gamma$ in validation. \texttt{mnist17} is created by extracting images for digits `1' and `7' from standard \texttt{mnist} dataset.}
\scalebox{0.9}{
    \begin{tabular}{ccccc}
    \toprule
    Name     & Task & $n$ & $d$ & $\gamma^*$ \\
    \midrule
    \texttt{cadata} & Regression & 8,000 & 8 & 1.0\\
    \texttt{E2006} & Regression & 19,227  & 150,360 & 1.0\\
    \texttt{mnist17} & Classification & 26,014  & 780 & 0.6 \\
    \texttt{rcv1} & Classification & 20,242 & 47,236 & 0.1\\
    \bottomrule
    \end{tabular}}
    \vspace{-40pt}
\end{wraptable}
\subsection{Experimental settings and baselines}
We conduct experiments on two regression and two binary classification datasets\footnote{Publicly available at \url{https://www.csie.ntu.edu.tw/~cjlin/libsvmtools/datasets/}}. The meta-information can be found in Table~\ref{tab:data}. We use a Gaussian kernel with width $\gamma$ to construct the graph. For each data, we randomly choose $n_l$ samples as the labeled set, and the rest are unlabeled. We normalize the feature vectors by  $x'\gets (x-\mu)/\sigma$, where $\mu$ is the sample mean, and $\sigma$ is the sample variance. For regression data, we also scale the output by $y'\gets (y-y_{\min})/(y_{\max}-y_{\min})$ so that $y'\in[0, 1]$. To evaluate the performance of label propagation models, for regression task we use RMSE metric defined as $\mathrm{RMSE}=\sqrt{\frac{1}{n_u}\sum_{i=1}^{n_u}(y_i-\hat{y}_i)^2}$, while for classification tasks we use error rate metric. For comparison with other methods, since \textbf{this is the first work on data poisoning attack to G-SSL}, we proposed several baselines according to graph centrality measures. The first baseline is random perturbation, where we randomly add Gaussian noise (for regression) or Bernoulli noise (for regression) to labels. The other two baselines based on graph centrality scores are more challenging, they are widely used to find the ``important'' nodes in the graph. Intuitively, we need to perturb ``important'' nodes to attack the model, and we decide the importance by node degree or PageRank. We explain the baselines with more details in the appendix.

\subsection{Effectiveness of data poisoning to G-SSL}
In this experiment, we consider the white-box setting where the attacker knows not only the ground truth labels $\vy_u$ but also the correct hyper-parameter $\gamma^*$. We thus apply our proposed label poisoning algorithms in Section~\ref{sec:regression} and~\ref{sec:classification} to attack regression and classification tasks, respectively. In particular, we apply $\ell_2$ constraint for perturbation $\vdelta_y$ in the regression task and use the greedy method in the classification task. The results are shown in Figure~\ref{fig:change_num_perturb},
\begin{figure}[h]
    \centering
    \includegraphics[width=\linewidth]{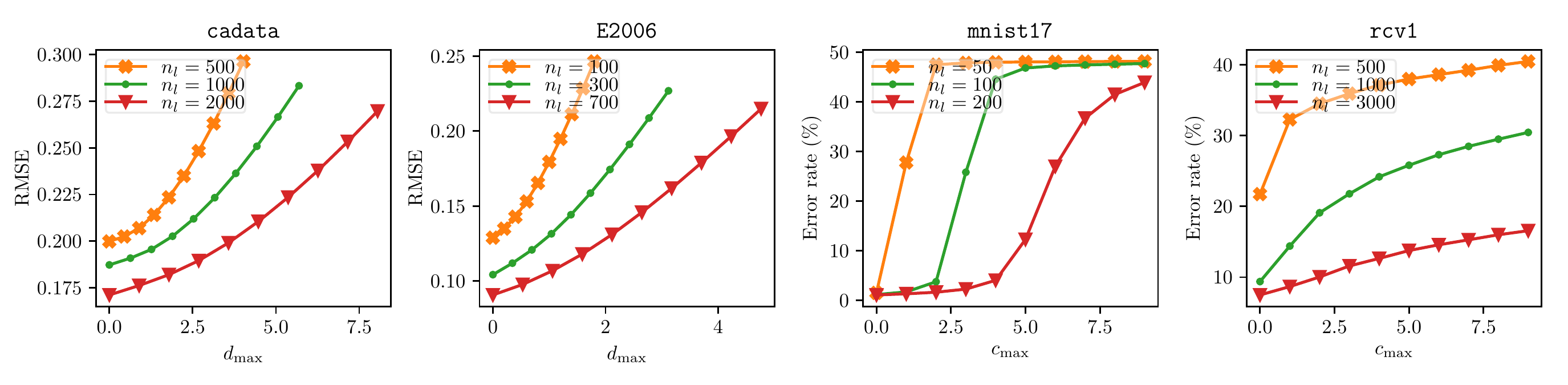}\\\vspace{-5pt}
    \includegraphics[width=\linewidth]{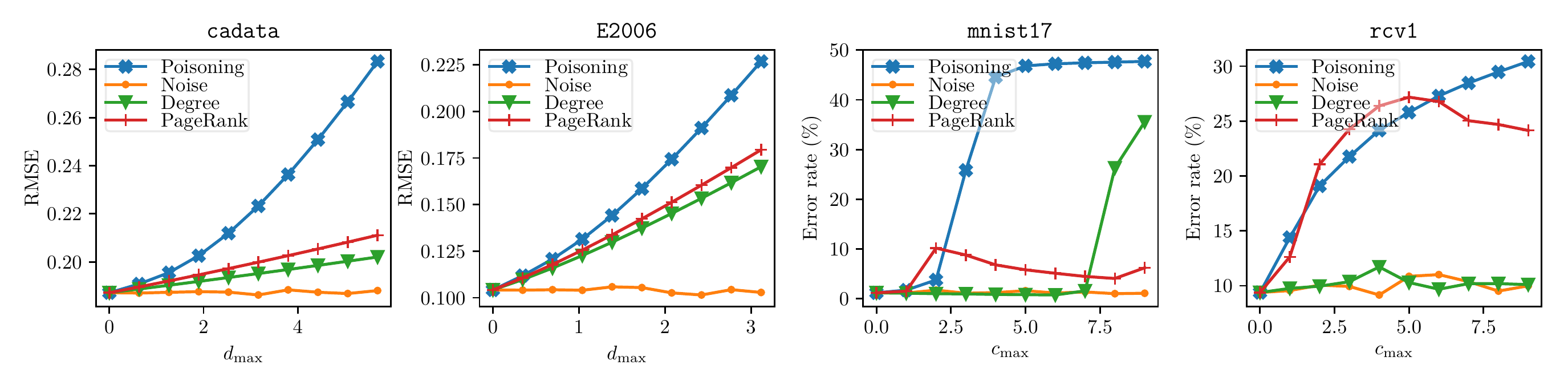}\vspace{-5pt}
    \caption{\textit{Top row:} testing the effectiveness of poisoning algorithms on four datasets shown in Table~\eqref{tab:data}. The left two datasets are regression tasks, and we report the RMSE measure. The right two datasets are classification tasks in which we report the error rate. For each dataset, we repeat the same attacking algorithm w.r.t. different $n_l$'s. \textit{Bottom row:} compare our poisoning algorithm with three baselines (random noise, degree-based attack, PageRank based attack). We follow our convention that $d_{\max}$ is the maximal $\ell_2$-norm distortion, and $c_{\max}$ is the maximal $\ell_0$-norm perturbation.}
    \label{fig:change_num_perturb}
\end{figure}
as we can see in this figure, for both regression and classification problems, small perturbations can lead to vast differences: for instance, on \texttt{cadata}, the RMSE increases from $0.2$ to $0.3$ when applied a carefully designed perturbation $\|\vdelta_y\|=3$ (this is very small compared with the norm of label $\|\vy_l\|\approx 37.36$); More surprisingly, on \texttt{mnist17}, the accuracy can drop from $98.46\%$ to $50\%$ by flipping just $3$ nodes. This phenomenon indicates that \textbf{current graph-based SSL, especially the label propagation method, can be very fragile to data poisoning attacks}. On the other hand, using different baselines (shown in Figure~\ref{fig:change_num_perturb}, bottom row), the accuracy does not decline much, this indicates that our proposed attack algorithms are more effective than centrality based algorithms.
\par
Moreover, the robustness of label propagation is strongly related to the number of labeled data $n_l$: for all datasets shown in Figure~\ref{fig:change_num_perturb}, we notice that the models with larger $n_l$ tend to be more resistant to poisoning attacks. This phenomenon arises because, during the learning process, the label information propagates from labeled nodes to unlabeled ones. Therefore even if a few nodes are ``contaminated'' during poisoning attacks, it is still possible to recover the label information from other labeled nodes. Hence this experiment can be regarded as another instance of ``no free lunch'' theory in adversarial learning~\cite{tsipras2018robustness}.

\subsection{\label{sec:unsupervised-attack-exp} Comparing poisoning with and without truth labels}
\begin{figure*}[h]
    \centering
    \includegraphics[width=0.97\linewidth]{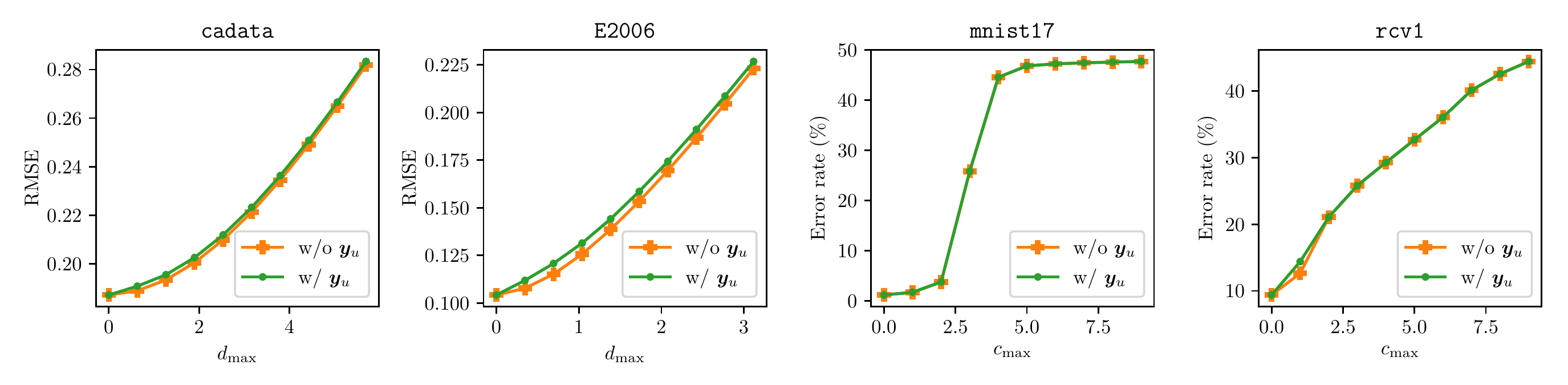}
    \caption{Comparing the effectiveness of label poisoning attack with and without knowing the ground truth labels of unlabeled nodes $\vy_u$. Interestingly, even if the attacker is using the estimated labels $\hat{\vy}_u$, the effectiveness of the poisoning attack does not degrade significantly.}
    \label{fig:compare_unsup}
\end{figure*}
We compare the effectiveness of poisoning attacks with and without ground truth labels $\vy_u$. Recall that if an attacker does not hold $\vy_u$, (s)he will need to replace it with the estimated values $\hat{\vy}_u$. Thus we expect a degradation of effectiveness due to the replacement of $\vy_u$, especially when $\hat{\vy}_u$ is not a good estimation of $\vy_u$. The result is shown in Figure~\ref{fig:compare_unsup}. Surprisingly, we did not observe such phenomenon: for regression tasks on \texttt{cadata} and \texttt{E2006}, two curves are closely aligned despite that attacks without ground truth labels $\vy_u$ are only slightly worse. For classification tasks on \texttt{mnist17} and \texttt{rcv1}, we cannot observe any difference, the choices of which nodes to flip are exactly the same (except the $c_{\max}=1$ case in \texttt{rcv1}). This experiment provides a valuable implication that hiding the ground truth labels cannot protect the SSL models, because the attackers can alternatively use the estimated ground truth $\hat{\vy}_u$.

\subsection{\label{sec:discrete_compare}Comparing greedy and probabilistic method}
\begin{figure}[h]
    \centering
    \includegraphics[width=0.7\linewidth]{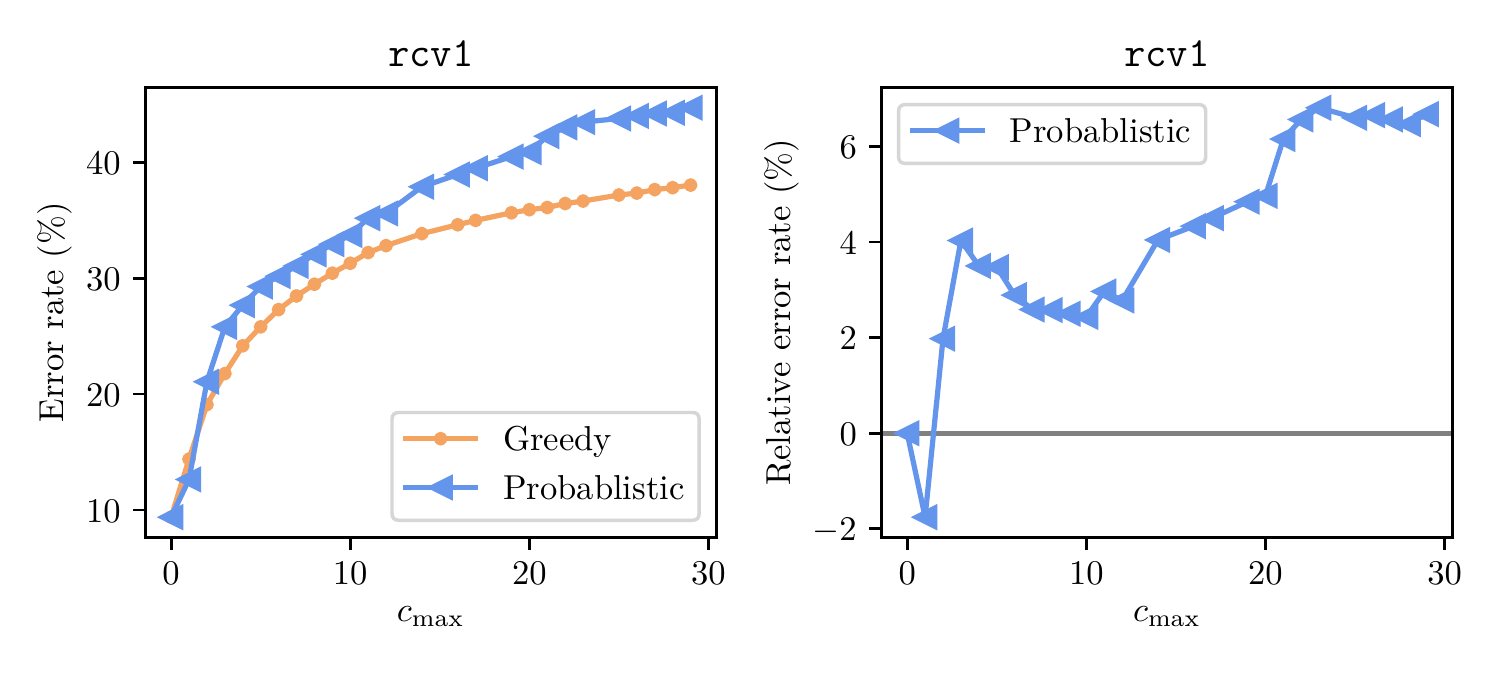}
    \caption{Comparing the relative performance of three approximate solvers to discrete optimization problem~\eqref{eq:obj-func-perturb_y-discrete}. For clarity, we also show the relative performance on the right (probabilistic $-$ greedy).}
    \label{fig:cmp_3methods}
\end{figure}
In this experiment, we compare the performance of three approximate solvers for problem~\eqref{eq:obj-func-perturb_y-discrete} in Section~\ref{sec:classification}, namely greedy and probabilistic methods. We choose \texttt{rcv1} data as oppose to \texttt{mnist17} data, because \texttt{rcv1} is much harder for poisoning algorithm: when $n_l=1000$, we need $c_{\max}\approx 30$ to make error rate $\approx 50\%$, whilst \texttt{mnist17} only takes $c_{\max}=5$. For hyperparameters, we set $c_{\max}=\{0,1,\dots,29\}$, $n_l=1000$, $\gamma^*=0.1$. The results are shown in Figure~\ref{fig:cmp_3methods}, we can see that for larger $c_{\max}$, greedy method can easily stuck into local optima and inferior than our probabilistic based algorithms.

\subsection{\label{sec:exp-imprecise-gamma}Sensitivity analysis of hyper-parameter} 
\begin{wrapfigure}{l}{0.45\textwidth}
  \begin{center}
    \includegraphics[width=\linewidth]{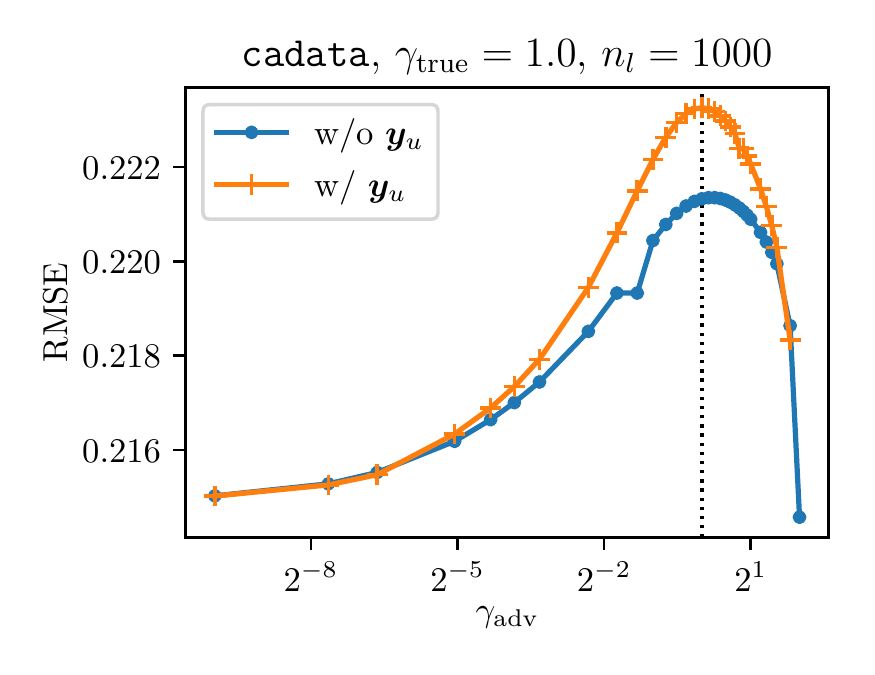}
    \caption{Experiment result on imperfect estimations of $\gamma^*$.}
    \label{fig:gamma-sensitivity}
  \end{center}
\end{wrapfigure}
Since we use the Gaussian kernel to construct the graph, there is an important hyper-parameter $\gamma$ (kernel width) that controls the structure of the graph defined in~\eqref{eq:label-prop-problem}, which is often chosen empirically by the victim through validation. Given the flexibility of $\gamma$, it is thus interesting to see how the effectiveness of the poisoning attack degrades with the attacker's imperfect estimation of $\gamma$. To this end, we suppose the victim runs the model at the optimal hyperparameter $\gamma=\gamma^*$, determined by validation, while the attacker has a very rough estimation $\gamma_{\mathrm{adv}}\approx\gamma^*$. We conduct this experiment on \texttt{cadata} when the attacker knows or does not know the ground truth labels $\vy_u$, the result is exhibited in Figure~\ref{fig:gamma-sensitivity}. 
It shows that when the adversary does not have exact information of $\gamma$, it will receive some penalties on the performance (in RMSE or error rate). However, it is entirely safe to choose a smaller $\gamma_{\mathrm{adv}}<\gamma_{\mathrm{truth}}$ because the performance decaying rate is pretty low. Take Figure~\ref{fig:gamma-sensitivity} for example, even though $\gamma_{\mathrm{adv}}=\frac{1}{8}\gamma_{\mathrm{truth}}$, the RMSE only drops from $0.223$ to $0.218$. On the other hand, if $\gamma_{\mathrm{adv}}$ is over large, the nodes become more isolated, and thus the perturbations are harder to propagate to neighbors.
\section{Conclusion}
We conduct the first comprehensive study of data poisoning to G-SSL algorithms, including label propagation and manifold regularization (in the appendix). The experimental results for regression and classification tasks exhibit the effectiveness of our proposed attack algorithms. In the future, it will be interesting to study poisoning attacks for deep semi-supervised learning models.

\section*{Acknowledgement}
Xuanqing Liu and Cho-Jui Hsieh acknowledge the support of NSF IIS-1719097, Intel faculty award, Google Cloud and Nvidia. Zhu acknowledges NSF 1545481, 1561512, 1623605, 1704117, 1836978 and the MADLab AF COE FA9550-18-1-0166.

\bibliographystyle{unsrt}
\bibliography{ref}
\newpage
\appendix

\section{Proof of Convergence}
We show that our gradient based nonconvex trust region solver is able to find a global minimum efficiently. First recall the objective function
\begin{equation}
    \label{eq:tr_problem}
    f(\vz^*)=\min_{\|\vz\|_2\le 1} f(\vz)=\frac{1}{2}\vz^\intercal \mH\vz+\vg^\intercal \vz, \quad \lambda_{\min}(\mH)<0.
\end{equation}
Suppose $\mH$ has decomposition $\sum_{i=1}^n\lambda_i\vv_i\vv_i^\intercal$ and rank $\lambda_1\le\dots\le\lambda_i\le 0\le\dots\le \lambda_n$. We only focus on ``easy case'': in which case $\vg^{(1)}=\vg^\intercal \vv_1\ne 0$ and $\vv_1$ is the corresponding eigenvector of $\lambda_1=\lambda_{\min}(\mH)$. In opposite, the hard case $\vg^{(1)}=0$ is hardly seen in practice due to rounding error, and to be safe we can also add a small Gaussian noise to $\vg$. To see the structure of solution, suppose the solution is $\vz^*$, then by KKT condition, we can get the condition of global optima
\begin{equation}
    \label{eq:KKT}
    \begin{aligned}
(\mH+\lambda\mI_d)\vz^*+\vg=0,\\
\lambda(1-\|\vz^*\|_2)=0,\\
\mH+\lambda\mI_d\succeq 0.
\end{aligned}
\end{equation}
By condition $\lambda_1<0$, further if $\vg^{(1)}\ne 0$, then $\lambda\ge -\lambda_1$ and $\|\vz^*\|=1$. Because $\vg^{(1)}\ne 0$, which implies $\lambda>-\lambda_1$. Immediately we know $\vz^*=-(\mH+\lambda\mI_d)^{-1}\vg$ and $\lambda$ is the solution of $\sum_{i=1}^n(\frac{\vg^{(i)}}{\lambda_i+\lambda})^2=1$.
\par
As a immediate application of \eqref{eq:KKT}, we can conclude the following lemma:
\begin{lemma}\label{le:glob-non-loc}
When $\vg^{(1)}\ne 0$ and $\lambda_1<0$, among all stationary points if $\vs^{(1)}\vg^{(1)}\le 0$ then $\vs=\vz^*$ is the global minimum.
\end{lemma}
\begin{proof}
We proof by contradiction. Suppose $\vs$ is a stationary point and $\vs^{(1)}\vg^{(1)}\le 0$, according to \eqref{eq:KKT} if $\vs$ is not a global minimum then the third condition in \eqref{eq:KKT} should be violated, implying that $\lambda_1+\lambda<0$. Furthermore, for stationary point $\vs$, we know the gradient of Lagrangian is zero: $(\mH+\lambda\mI_d)\vz^*+\vg=0$. Projecting this equation onto $\vv_1$ we get 
\begin{equation}\label{eq:project-gradient-v1}
    (\lambda_1+\lambda)\vs^{(1)}+\vg^{(1)}=0.
\end{equation}
By condition $\vg^{(1)}\ne 0$ we know $\vs^{(1)}\ne 0$; multiply both sides of Eq.~\ref{eq:project-gradient-v1} by $\vs^{(1)}$ we get $\vs^{(1)}\vg^{(1)}>0$, which is in contradiction to $\vs^{(1)}\vg^{(1)}\le 0$.
\end{proof}
We now consider the projected gradient descent update rule $\vz_{t+1}=\text{Prox}_{\|\cdot\|_2}(\vz_t-\eta\nabla f(\vz_t))$, with following assumptions:
\begin{assumption}{(Bounded step size)}\label{as:step_size}
Step size $\eta<1/\beta$, where $\beta=\|\mH\|_{\mathrm{op}}$.
\end{assumption}
\begin{assumption}{(Initialize)}\label{as:init}
$\vz_0=-\alpha \frac{\vg}{\|\vg\|}$, $0<\alpha<\min\big(1, \frac{\|\vg\|^3}{|\vg^\intercal\mH\vg|}\big)$.
\end{assumption}

Under these assumptions, we next show proximal gradient descent converges to global minimum
\begin{theorem}\label{th:converge}
Under proximal gradient descent update: $\vz_{t+1}=\text{Prox}_{\|\cdot\|_2}\big(\vz_t-\eta\nabla f(\vz_t)\big)$, and Assumption \ref{as:step_size} if $\vz_t^{(i)}\vg^{(i)}\le 0$ then $\vz_{t+1}^{(i)}\vg^{(i)}\le 0$. Combining with Assumption \ref{as:init} and Lemma \ref{le:glob-non-loc}, if $\lambda_1<0$, $\vg^{(1)}\ne 0$ then $\vz_t$ converges to a global minimum $\vz^*$.
\end{theorem}
\begin{proof}
Notice the projection onto sphere will not change the sign of $\vz_{t+1}^{(i)}$, so:
$$
\sign\big(\vz_{t+1}^{(i)}\vg^{(i)}\big)=\sign\big((1-\eta \lambda_i)\vz_t^{(i)}\vg^{(i)}-\eta \vg^{(i)2}\big)
$$
$\eta_t<1/|\lambda_n|$ ensures $1-\eta\lambda_i>0$ for all $i\in [n]$. From Assumption~\ref{as:init} we know $\vz_0^{(i)}\vg^{(i)}=-\alpha\frac{\vg^{(i)2}}{\|\vg\|}\le 0$, so $\vz_t^{(i)}\vg^{(i)}\le 0$ for all $t$. We complete the proof by combining it with Lemma~\ref{le:glob-non-loc}.
\end{proof}

By careful analysis, we can actually divide the convergence process into two stages. In the first stage, the iterates $\{\vz_t\}$ stay inside the sphere $\|\vz_t\|<1$; in the second stage the iterates stay on the unit ball $\|\vz_t\|=1$. Furthermore, we can show that the first stage ends with finite number of iterations. Before that, we introduce the following lemma:
\begin{lemma}
Considering the first stage, when iterates $\{\vz_t\}$ are inside unit sphere $\|\vz_t\|_2\le 1$, i.e. under the update rule $\vz_{t+1}=\vz_t-\eta\nabla f(\vz_t)$, and under assumption that $\vz_t^\intercal\nabla f(\vz_t)\le 0$, we will have $\vz_t^\intercal \mH\nabla f(\vz_t)\ge \beta \vz_t^\intercal\nabla f(\vz_t)$ (recall we define $\beta$ as the operator norm of $\mH$).
\end{lemma}
\begin{proof}
We first define $\vw_t^{(i)}=\vz_t^{(i)}/(-\eta \vg^{(i)})$, then by iteration rule $\vz_{t+1}=(\mI-\eta\mH)\vz_t-\eta\vg$, projecting both sides on $\vv_i$,
\[
\vz_{t+1}^{(i)}=(1-\eta\lambda_i)\vz_t^{(i)}-\eta\vg^{(i)},
\]
dividing both sides by $-\eta\vg^{(i)}$, we get
\begin{equation}\label{eq:z-iter}
\vw_{t+1}^{(i)}=(1-\eta\lambda_i)\vw_t^{(i)}+1,
\end{equation}
solving this geometric series, we get:
\begin{equation}\label{eq:z_formula}
\vw_t^{(i)}=(1-\eta\lambda_i)^t\big(\vw_0^{(i)}-\frac{1}{\eta\lambda_i}\big)+\frac{1}{\eta\lambda_i},
\end{equation}
suppose at $t$-th iteration we have $\vw_t^{(i)}\ge \vw_{t+1}^{(i)}$, after plugging in Eq.~\eqref{eq:z_formula} and noticing $0<\eta\lambda_i<1$
\begin{equation}\label{eq:z0-condition}
\vw_0^{(i)}-\frac{1}{\eta\lambda_i}\ge (1-\eta\lambda_i)\big(\vw_0^{(i)}-\frac{1}{\eta\lambda_i}\big).
\end{equation}
Furthermore, from Assumption~\ref{as:init} we know that $\vw_0^{(i)}=\vz_t^{(i)}/(-\eta \vg^{(i)})=\alpha\frac{1}{\eta\|\vg\|_2}\ge 0$, if we have $\vw_0^{(i)}-\frac{1}{\eta\lambda_i}\le 0$, equivalently $0<\eta\lambda_i\le 1/\vw_0^{(i)}$ then by Eq.~\eqref{eq:z0-condition} we know $1-\eta\lambda_i\ge 1\Leftrightarrow \eta\lambda_i\le 0$ leading to a contradiction, so it must hold that 
\[
\vw_0^{(i)}-\frac{1}{\eta\lambda_i}>0 \text{ and } 1-\eta\lambda_i\le 1.
\]
\par
At the same time, the eigenvalues are nondecreasing, $\lambda_j\ge \lambda_i$ for $j\ge i$, which means 
\begin{equation}\label{eq:lam_j}
1-\eta\lambda_j\le 1.
\end{equation}
Also recalling the initialization condition implies $\vw_0^{(j)}=\alpha\frac{1}{\eta\|\vg\|_2}=\vw_0^{(i)}$, subtracting both sides by $1/\eta\lambda_j$ and noticing $\lambda_j\ge\lambda_i$
\begin{equation}\label{eq:positive_j}
\vw_0^{(j)}-\frac{1}{\eta\lambda_j}=\vw_0^{(i)}-\frac{1}{\eta\lambda_j}\ge \vw_0^{(i)}-\frac{1}{\eta\lambda_i}>0.
\end{equation}
Combining Eq.~\eqref{eq:lam_j} with Eq.~\eqref{eq:positive_j}, we can conclude that if $\vw_t^{(i)}\ge\vw_{t+1}^{(i)}$ holds, then such relation also holds for index $j>i$
\begin{equation}\label{eq:transmitting}
\vw_0^{(j)}-\frac{1}{\eta\lambda_j}\ge (1-\eta\lambda_j)\big(\vw_0^{(j)}-\frac{1}{\eta\lambda_j}\big) \Longleftrightarrow \vw_t^{(j)}\ge \vw_{t+1}^{(j)} \text{ for } j\ge i.
\end{equation}
Consider at any iteration time $t$, suppose $i^*\in[n]$ is the smallest coordinate index such that $\vw_t^{(i^*)}\ge \vw_{t+1}^{(i^*)}$, and hence $\vw_t^{i}<\vw_{t+1}^{i}$ holds for all $i<i^*$. By Eq.~\eqref{eq:transmitting} we know  and $\vw_t^{i}\ge \vw_{t+1}^{i}$ for any $i\ge i^*$ (such a $i^*$ may not exist, but it doesn't matter). By analyzing the sign of $\vz_t$ we know:
$$
\sign\Big(\vz_t^{i}\big(\vz_t^{(i)}-\vz_{t+1}^{(i)}\big)\Big)=\sign\Big(\vw_t^{(i)}\big(\vw_t^{(i)}-\vw_{t+1}^{(i)}\big)\Big)=\sign\big(\vw_t^{(i)}-\vw_{t+1}^{(i)}\big), \forall i\in[n],
$$
the second equality is true due to Eq.~\eqref{eq:z-iter}, we know $\vw_t^{(i)}>0$ for all $i$ and $t$. 
\par
We complete the proof by following inequalities:
\begin{equation}
\begin{aligned}
\vz_t^\intercal \mA\nabla f(\vz_t)
&=\frac{1}{\eta}\sum_{i=1}^{i^*-1}\underbrace{\lambda_i\vz_t^{(i)}(\vz_t^{(i)}-\vz_{t+1}^{(i)})}_{\le 0}+\frac{1}{\eta}\sum_{i=i^*}^n\underbrace{\lambda_i\vz_t^{(i)}(\vz_t^{(i)}-\vz_{t+1}^{(i)})}_{\ge 0}\\
&\ge \frac{\lambda_{i^*-1}}{\eta}\sum_{i=1}^{i^*-1}\vz_t^{(i)}(\vz_t^{(i)}-\vz_{t+1}^{(i)})+\frac{\lambda_{i^*}}{\eta}\sum_{i=i^*}^n\vz_t^{(i)}(\vz_t^{(i)}-\vz_{t+1}^{(i)})\\
&\ge \frac{\lambda_{i^*}}{\eta}\sum_{i=1}^{n}\vz_t^{(i)}(\vz_t^{(i)}-\vz_{t+1}^{(i)})\\
&\ge \beta \vz_t^\intercal\nabla f(\vz_t).
\end{aligned}
\end{equation}
Where the last inequality follows from assumption in this lemma.
\end{proof}
By applying this lemma on the iterates $\{\vz_t\}$ that are still inside the sphere, we will eventually conclude that $\|\vz_t\|_2$ monotone increases. In fact, we have the following theorem:
\begin{theorem}
Suppose $\{\vz_t\}$ is in the region $\|\vz_t\|_2<1$, such that proximal gradient update equals to plain GD: $\vz_{t+1}=\vz_t-\eta\nabla f(\vz_t)$, then under this update rule, $\|\vz_t\|$ is monotone increasing.
\end{theorem}
\begin{proof}
We prove by induction. First of all, notice $\|\vz_{t+1}\|^2=\|\vz_t\|^2-2\eta \vz_t^\intercal\nabla f(\vz_t)+\eta^2\|\nabla f(\vz_t)\|^2$, to prove $\|\vz_{t+1}\|^2\ge\|\vz_t\|^2$, it remains to show $\vz_t^\intercal\nabla f(\vz_t)\le 0$. For $t=0$ we note that
\begin{equation}
    \label{eq:t=0}
    \vz_0^\intercal\nabla f(\vz_0)=\alpha^2\frac{\vg^\intercal\mH\vg}{\|\vg\|^2}-\alpha\|\vg\|\le 0,
\end{equation}
where the last inequality follows from Assumption~\ref{as:init}. Now suppose $\vz_{t-1}^\intercal\nabla f(\vz_{t-1})\le 0$ and by update rule $\vz_t=\vz_{t-1}-\eta\nabla f(\vz_{t-1})$ we know:
$$
\begin{aligned}
\vz_t^\intercal\nabla f(\vz_t)= &\vz_{t-1}^\intercal\nabla f(\vz_{t-1})-\eta\|\nabla f(\vz_{t-1})\|^2-\underbrace{\eta \vz_{t-1}^\intercal A\nabla f(\vz_{t-1})}_{(1)}\\
&\quad +\underbrace{\eta^2\nabla f(\vz_{t-1})^\intercal A\nabla f(\vz_{t-1})}_{(2)}.
\end{aligned}
$$
From Lemma~4 we know $(1)\ge \beta \vz_{t-1}\nabla f(\vz_{t-1})$ and recall $\beta$ is the operator norm of $A$, we have $(2)\le \beta\| \nabla f(\vz_{t-1})\|^2$, combining them together:
\begin{equation}
\vz_t^\intercal \nabla f(\vz_t)\le (1-\beta\eta)\vz_{t-1}^\intercal \nabla f(\vz_{t-1})-\eta(1-\eta\beta)\|\nabla f(\vz_t)\|^2,
\end{equation}
by choosing $\eta<1/\beta$ we proved $\vz_t^\intercal \nabla f(\vz_t)\le 0$. 
\par
Due to induction rule we know that $\vz_t^\intercal\nabla f(\vz_t)\le 0$ holds for all $t$ and moreover, $\|\vz_{t+1}\|$ is monotone increasing.
\end{proof}
We can easily improve the results above, to show that phase I (where $\|\vz_t\|\le 1$) will eventually terminate after finite number of iteration. This is formally described in the following proposition:
\begin{proposition}{(Finite phase I)}
Assuming $\lambda_1<0$, suppose $t^*$ is the index that $\|\vz_{t^*}\|<1$ and $\|\vz_{t^*+1}\|\ge 1$, then $t^*$ is bounded by:
\begin{equation}\label{eq:bound-t}
\begin{aligned}
t^*\le \log(1-\eta\lambda_1)^{-1}\Big[\log\big(\frac{1}{\eta|\vg^{(1)}|}-\frac{1}{\eta\lambda_1}\big)
-\log\big(\frac{-\vz_0^{(1)}}{\eta \vg^{(1)}}-\frac{1}{\eta\lambda_1}\big)\Big].
\end{aligned}
\end{equation}
\end{proposition}
\begin{proof}
This directly follows from:
$$
1\le \eta^2\vg^{(1)2}\vw_{t^*+1}^{(1)2}=\vz_{t^*+1}^{(1)2}\le \|\vz_{t^*+1}\|^2,
$$
together with Eq.~\eqref{eq:z_formula} immediately comes to Eq.~\eqref{eq:bound-t}.
\end{proof}
Lastly, it remains to show the converge rate in phase II, this is actually a standard manifold gradient descent problem

\begin{theorem}
Let $\{\vz_t\}$ be an infinite sequence of iterates generated by line search gradient descent, then every accumulation point of $\{\vz_t\}$ is a stationary point of the cost function $f$. 
\end{theorem}
\begin{theorem}\label{th:linear-converge}
Let $\{\vz_t\}$ be an infinite sequence of iterates generated by line search gradient descent, suppose it converges to $\vz^*$. Let $\lambda_{\mH,\min}$ and $\lambda_{\mH,\max}$ be the smallest and largest eigenvalues of the Hessian at $\vz^*$. Assume that $\vz^*$ is a local minimizer then $\lambda_{\mH,\min}>0$ and given $r$ in the interval $(r_*, 1)$ with $r_*=1-\min\big(2\sigma\bar{\alpha}\lambda_{\mH,\min}, 4\sigma(1-\sigma)\beta\frac{\lambda_{\mH,\min}}{\lambda_{\mH,\max}}\big)$, there exists an integer $K$ such that:
$$
f(\vz_{t+1})-f(\vz^*)\le r\big(f(\vz_t)-f(\vz^*)\big),
$$
for all $t\ge K$.
\end{theorem}
\begin{proof}
See Theorem 4.3.1 and Theorem 4.5.6 in [Absil et al., 2009].
\end{proof}
To apply Theorem~\ref{th:linear-converge}, we need to check $\lambda_{\mH, \min}>0$. To do that we can directly calculate its value, by the definition of Riemanndian Hessian, we have
\begin{equation}
\label{eq:Riemannian-hess}
\begin{aligned}
\text{Hess}\ f(\vx)&=\text{Hess}(f\circ \text{Exp}_{\vx})(0_{\vx}),\\
\langle\text{Hess}\ f(\vx)[\xi], \xi\rangle&=\langle\text{Hess}\ (f\circ \text{Exp}_{\vx})(0_{\vx})[\xi], \xi\rangle.
\end{aligned}
\end{equation}
Then for $\xi\in T_x\mathcal{M}$,
\begin{equation}
\langle\text{Hess}\ f(\vx)[\xi], \xi\rangle=\langle\text{Hess}\ (f\circ R_{\vx})(0_{\vx})[\xi], \xi\rangle=\eval[2]{\od[2]{}{t}f(R_{\vx}(t\xi))}_{t=0},
\end{equation}
we then expand $f(R_{\vx}(t\xi))$ to,
\begin{equation}
f(R_{\vx}(t\xi))=\frac{\xi^\intercal \mH\xi\cdot t^2+\xi^\intercal \mH{\vx}\cdot t+{\vx}^\intercal \mH\vx}{2\|\vx+t\xi\|_2^2}+\frac{\vg^\intercal \vx+\vg^\intercal \xi\cdot t}{\|\vx+t\xi\|_2^2}.
\end{equation}
By differentiating $t$ twice and set $t=0$ (this can be done by software), we finally get
\begin{equation}
    \langle\text{Hess}\ f(\vx)[\xi], \xi\rangle=-\vx^{\intercal}\mH\vx+\xi^\intercal \mH\xi-\vg^\intercal \vx.
\end{equation}
Taking $\vx=\vz^*$ into above equation, we get 
\[
\langle\text{Hess}\ f(\vz^*)[\xi], \xi\rangle=-\vz^{*{\intercal}}\mH\vz^*+\xi^\intercal \mH\xi-\vg^\intercal \vz^*
\]
On the other hand, by optimal condition \eqref{eq:KKT}, we have:
\begin{equation}
    \label{eq:convex}
    \vz^{*\intercal}(\mH+\lambda\mI_d)\vz^*+\vz^{*\intercal} \vg=0\Longrightarrow -\vz^{*\intercal} H\vz^*-\vg^\intercal \vz^*=\lambda,
\end{equation}
so $\langle\text{Hess}f(\vz^*)[\xi], \xi\rangle=\lambda+\xi^\intercal \mH\xi\ge \lambda+\lambda_1\overset{!}{>} 0$. Where $\overset{!}{>}$ is guaranteed by gradient condition in \eqref{eq:KKT}:
\begin{equation}
(\lambda_1+\lambda)\vz^{*{(1)}}+\vg^{(1)}=0,
\end{equation}
in ``easy-case'', $\vg^{(1)}\ne 0$, so $\lambda_1+\lambda\ne 0$ and Hessian condition in \eqref{eq:KKT} can be improved to $\lambda_1+\lambda>0$.
\par
Based on above discussion, we know $\lambda_{H,\min}\ge \lambda+\lambda_1>0$ and $\lambda_{H,\max}\le \lambda+\lambda_n$.

\section{Supplementary Experiments on Trust Region Solver}
We sample an indefinite random matrix by $\mH=\mB\mB^\intercal-\lambda \mI_n$, where $\mB\in\mathbb{R}^{n\times (n-1)}$ and $\mB_{ij}\overset{\text{iid}}{\sim}\mathcal{N}(0,1)$, obviously $\lambda_{\min}(\mH)=\lambda_1=-\lambda$. Afterwards we sample a vector $\vg$ by $\vg_i\overset{\text{iid}}{\sim}\mathcal{N}(0,1)$. 
it is totally fine to ignore the hard case, because the probability is zero. 
By changing the value of $\lambda$ in $\{10,30,50,70,90,110\}$, we plot the function value decrement with respect to number of iterations in Figure~\ref{fig:tr_experiment}(left). As we can see, the iterates first stay inside of the sphere (phase I) for a few iterations and then stay on the boundary (phase II). To inspect how $\lambda$ changes the duration of phase I, we then plot the number of iterations it takes to reach phase II, under different $\lambda$ values shown in Figure~\ref{fig:tr_experiment}(right). Recall in \eqref{eq:bound-t}, number of iterations is bounded as a function of $\lambda$, which can be further simplified to:
\begin{equation}
    \label{eq:simple-niter}
    t^*\le \frac{\log(1+\frac{\lambda}{|\vg^{(1)}|})}{\log(1+\eta\lambda)}=\frac{\log(1+c_1\lambda)}{\log(1+c_2\lambda)},
\end{equation}
where we set $\vz_0^{(1)}=0$ to simplify the formula. By fitting the data point with function $T(\lambda)=\frac{\log(1+c_1\lambda)}{\log(1+c_2\lambda)}$, we find our bounds given by Lemma~6 is quite accurate.

\begin{figure*}[htb]
    \centering
\includegraphics[width=0.7\linewidth]{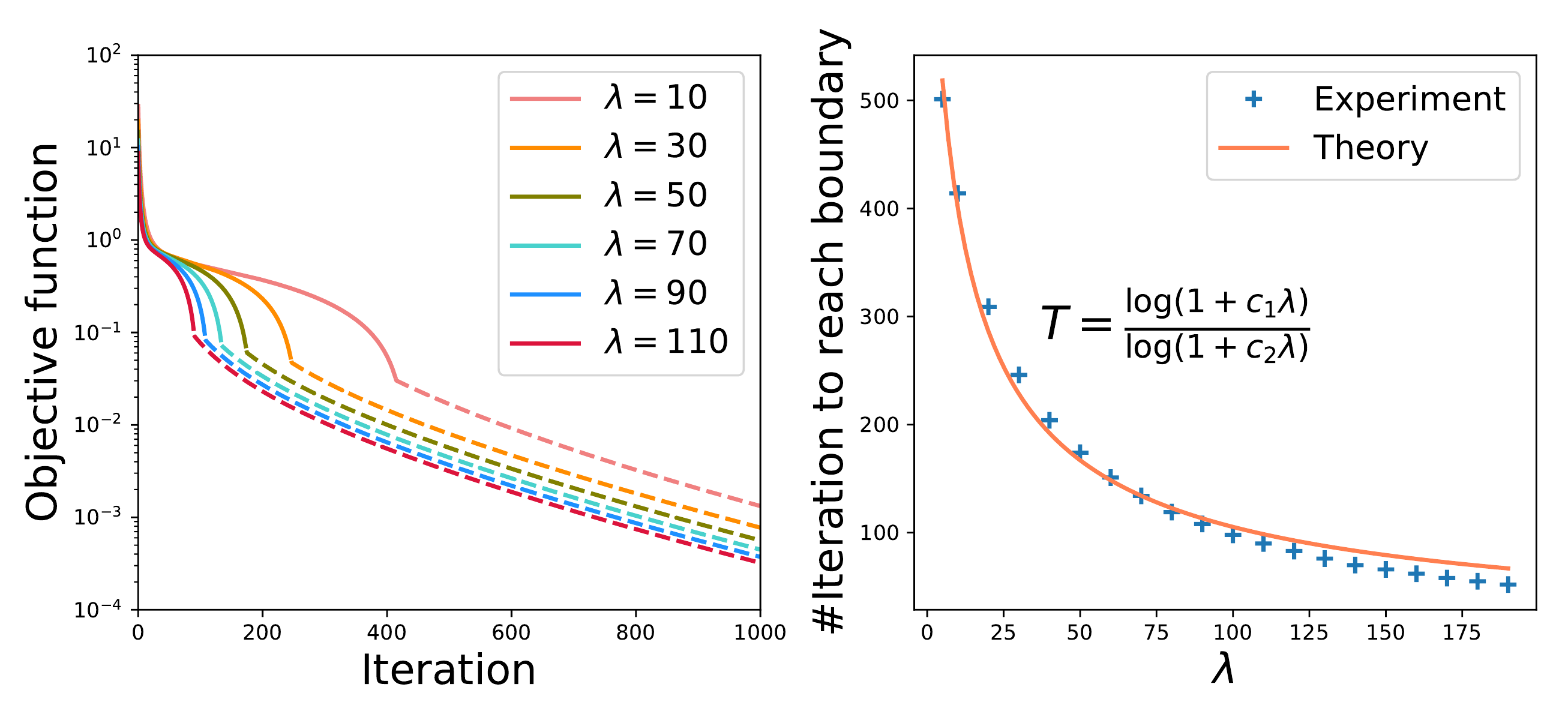}
    \caption{\textit{Left}: Trust region experiment, we use solid lines to indicate iterations inside the sphere and dash lines to indicate iterations on the sphere. By changing $\lambda$ we can modify the function curvature. \textit{Right}: \#Iteration it takes to reach sphere under different $\lambda$'s, we also fit the curve by model $T=\frac{\log(1+c_1\lambda)}{\log(1+c_2\lambda)}$ derived in Eq.~\eqref{eq:bound-t}.}
    \label{fig:tr_experiment}
\end{figure*}
\section{Baselines}
There are three baselines included in the experiments, namely random noise, degree-based poisoning and PageRank-based poisoning. The first baseline, random noise, is the simplest one. For continuous label, the perturbation is created by $\vdelta_{\vy}=d_{\max}\frac{\epsilon}{\|\epsilon\|_2}$ with $\epsilon\sim\mathcal{N}(0, 1)$; for discrete label, we randomly choose $c_{\max}$ indices $\{k_1, k_2, \dots, k_{c_{\max}}\}$ from $\{1,2,\dots, n\}$ and then set $\vdelta_{\vy}[{k_i}]=-1$.
 \par
 As to degree-based poisoning, we first calculate the degree vector $\mathrm{deg}[i]=\sum_{j=1}^n\mS_{ij}$ of all nodes, then for continuous label we load the perturbation weighted by degree, i.e. $|\vdelta_{\vy}[i]|=d_{\max}\sqrt{\frac{\mathrm{deg}^2[i]}{\sum_{j=1}^n\mathrm{deg}^2[j]}}$, and the sign of $\vdelta_{\vy}[i]$ is determined by the gradient of loss on $\vdelta_{\vy}$ at $\vdelta_{\vy}=\bm{0}$. Specifically:
 \[
 \sign(\vdelta_{\vy})=\sign\left(\frac{\partial L(\vdelta_{\vy})}{\partial \vdelta_{\vy}}|_{\vdelta_{\vy}=0}\right),
 \]
this makes sure that the direction is good enough to increase the prediction loss of SSL models. For discrete label, we simply choose the largest $c_{\max}$ training labels to flip: $\vdelta_{\vy}[i]=-1$ if and only if node-$i$ has many neighboring nodes. This is to maximize the influence of perturbations.
 \par
Similarly, we can also a PageRank based poisoning attack, the only difference is that we use PageRank to replace the degree score.

 \section{Supplementary Experiments on Data Poisoning Attacks}
In this section, we design more experiments on other cases that are not able show up in the main text. The problem settings and datasets are the same as previous experiments.
\subsection{Sparse and group sparse constraint}
In reality, the adversary may only be able to perturb very small amount of data points, this requirement renders sparse constraint. In specific, we consider
\begin{equation}
\begin{aligned}\label{eq:sparse-constraints}
\mathcal{R}_1&=\{\|\vdelta_y\|_0\le c_{\max}\text{ and } \|\vdelta_y\|_2\le d_{\max}\}, \\
\mathcal{R}_2&=\Big\{\sum_{i=1}^n\sI\{\Delta_x[i,:]\ne \bm{0}\}\le c_{\max}\Big\}.
\end{aligned}
\end{equation}
The constraint on $\mathcal{R}_2$ implies that only a limited number of data can be perturbed, so we enforce a row-wise group sparsity. Both $\mathcal{R}_1$ and $\mathcal{R}_2$ can be added to regression/classification tasks, below we take regression task as an example to show the effectiveness.
\par
For $\mathcal{R}_1$, the optimization problem is essentially a sparse PCA problem 
\begin{equation}
    \label{eq:SparsePCA}
    \begin{aligned}
        \min_{\vdelta_y}\ &-\frac{1}{2}\Big\|(\mD_{uu}-\mS_{uu})^{-1}\mS_{ul}\vdelta_y\Big\|^2_2 \\ 
        \texttt{s.t.}
        &\ \ \|\vdelta_y\|_0\le c_{\max}, \
         \ \|\vdelta_y\|_2\le d_{\max}.
    \end{aligned}
\end{equation}
For this kind of problem, many efficient solvers were proposed during the past decades, including threshold method~\cite{cadima1995loading}, LASSO based method~\cite{zou2006sparse}, or by convex relaxation~\cite{d2008optimal}.

\begin{figure}[t]
    \centering
    \vspace{-5pt}
    \includegraphics[width=0.49\linewidth]{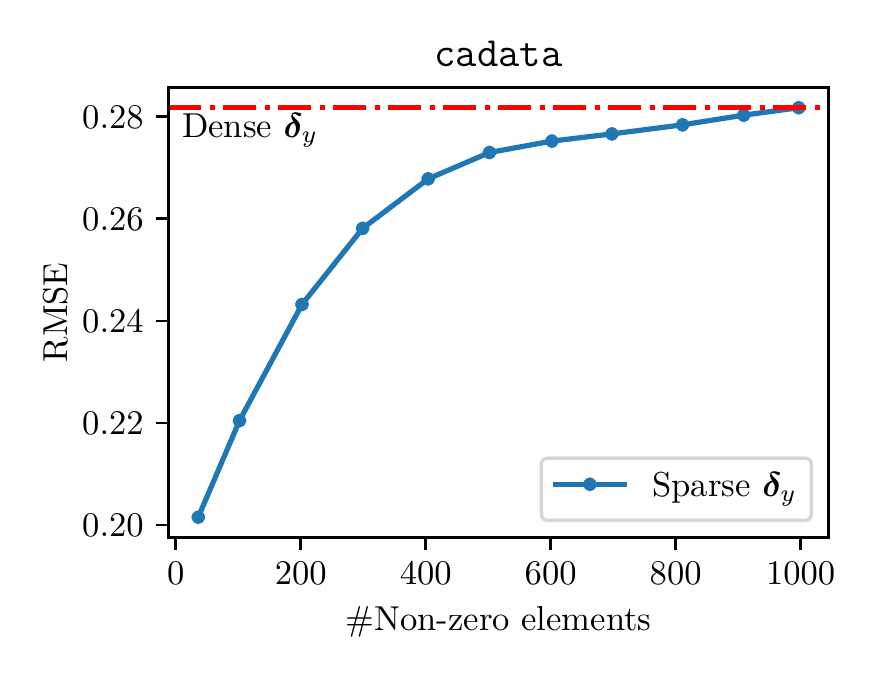}
    \includegraphics[width=0.49\linewidth]{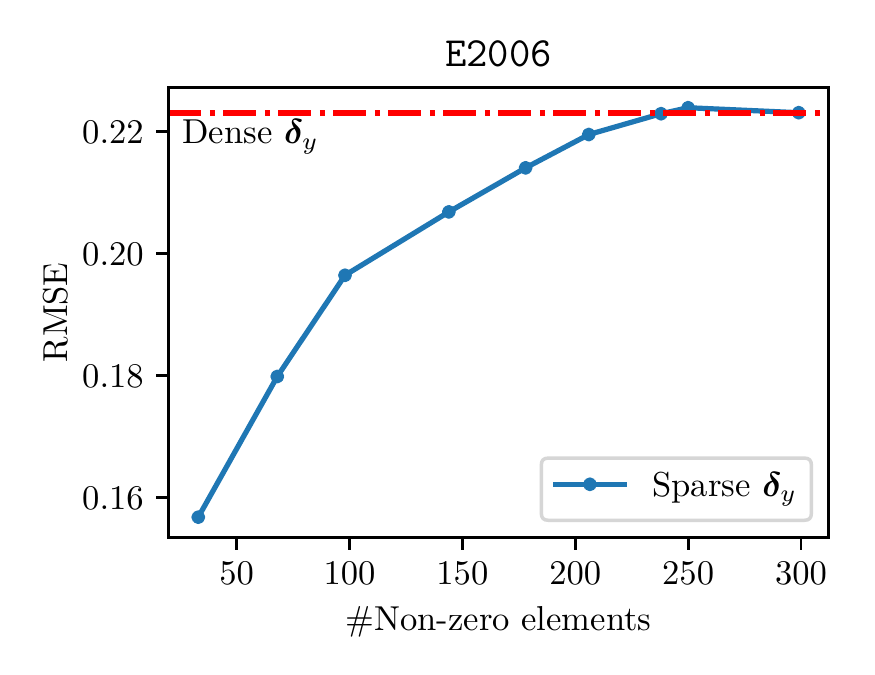}

    \caption{The effectiveness of $\ell_0/\ell_2$-mixed constraints in finding the sparse perturbations $\vdelta_y$. For \texttt{cadata}, we set $n_l=1000$, while for \texttt{E2006} data $n_l=300$. The RMSE results of dense solutions~\eqref{eq:estimate-label} are marked with red dashed lines. }

    \label{fig:sparse_pca}
\end{figure}

In order to solve the sparse PCA problem, we adopt the LASSO based sparse PCA solver~\cite{zou2006sparse}. We conduct the experiment on \texttt{cadata} and \texttt{E2006} data, then plot the  sparsity (measured by \#nnz of $\vdelta_y$) and corresponding RMSE in Figure~\ref{fig:sparse_pca}. For comparison, we also include the RMSE when no $\ell_0$ sparsity constraint is enforced. Interestingly,
we observe that the RMSE increases rapidly as $\vdelta_y$ is relatively sparse, and later it gradually stabilizes before reaching the same RMSE of dense solution. That is to say, when attackers have constraint on the maximal number of perturbation they could make,  our sparse PCA based solution is able to make a good trade-off between sparsity and RMSE.
\par
\begin{figure}
    \centering
    \includegraphics[width=0.5\linewidth]{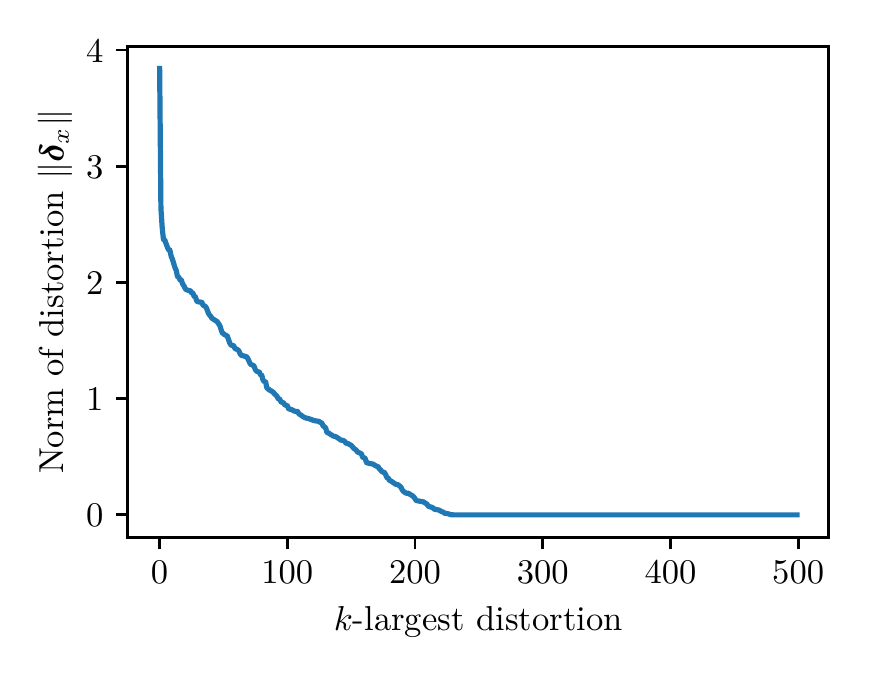}
    \caption{We ranked the perturbations $\vDelta_l$ by their $\ell_2$-norm, to see the decay rate of distortion.}
    \label{fig:norm_decay}
\end{figure}
For $\mathcal{R}_2$, the har dconstraint is replaced with a group LASSO regularizer $\lambda \sum_{i=1}^{n_l}\|\vDelta_l[i,:]\|_2$ and we use proximal gradient descent to solve the optimization problem. By changing the hyper-parameter $\lambda$ we can indirectly change the group sparsity of $\Delta_x$. As above, we run the experiment on \texttt{mnist17} data, result is shown in Figure~\ref{fig:norm_decay}.

\subsection{Data poisoning attack on manifold regularization model}
Apart from label propagation model for semi-supervised learning, our method can also be seamlessly applied to manifold regularization method~\cite{belkin2006manifold}. Manifold regularization based SSL solves the following optimization problem
\begin{equation}
    \label{eq:manifold-regularization-method}
    f^*=\mathop{\arg\min}_{f\in\mathcal{F}}\sum_{i=1}^{n_l}\ell\big(f(\vx_i), \vy_i\big)+\lambda\|f\|^2+\beta\sum_{i,j=1}^{n_l+n_u}\mS_{ij}\big(f(\vx_i)-f(\vx_j)\big)^2,
\end{equation}
where $\mathcal{F}$ is the set of model functions. $(\vx_i,\vy_i)$ is the $i$-th data pair in $(\mX,\vy)$, $\ell(\hat{y}, y)$ is the loss function. The model family $\mathcal{F}$ ranges from linear models $f(\vx)=\vw^\intercal\vx$ to very complex deep neural networks. If we limit our scope to the linear case, then \eqref{eq:manifold-regularization-method} has a closed form solution:
\begin{equation}
    \label{eq:solution-manifold-regularization}
        \vw^*=(\mX_l^\intercal\mX_l+\lambda\mI+\beta\mX^\intercal\mL\mX)^{-1}\mX_l^\intercal\vy_l=\mP\vy_l
    \end{equation}
where $\mX_l$ is the feature matrix of all labeled nodes, $\mX$ is the feature matrix of labeled and unlabeled nodes, $\mL$ is the graph Laplacian. 
\par
Manifold regularization term in \eqref{eq:solution-manifold-regularization} enforces two nodes that are close to each other (i.e. large $\mS_{ij}$) to hold similar labels, and that is similar to the objective of label propagation. This motivates us to extend our algorithms for attacking label propagation to attack manifold regularization based SSL. As an example, we discuss poisoning attack to manifold regularization model for regression task, where the problem can be formulated as
\begin{equation}
    \label{eq:obj-func-attack-manifold}\ 
    \min_{\|\vdelta_y\|\le d_{\max}}-\frac{1}{2}\|\mX_u\mP(\vy_l+\vdelta_y)-\vy_u\|^2_2.
\end{equation}
Clearly, Eq.~\eqref{eq:obj-func-attack-manifold} is again a non-convex trust region problem, and we can apply our trust region problem solver to it. For experiment, we take regression task on \texttt{cadata} as an example, different from label propagation, manifold regularization learns a parametric model $f_{\vw}(\vx)$ that is able to generalize to unseen graph. So for manifold regularization we can do label poisoning attack in both transductive and inductive settings. The experiment result is shown in Figure~\ref{fig:manifold-regularization}. 
\begin{figure}[htb]
    \centering
    \includegraphics[width=0.6\linewidth]{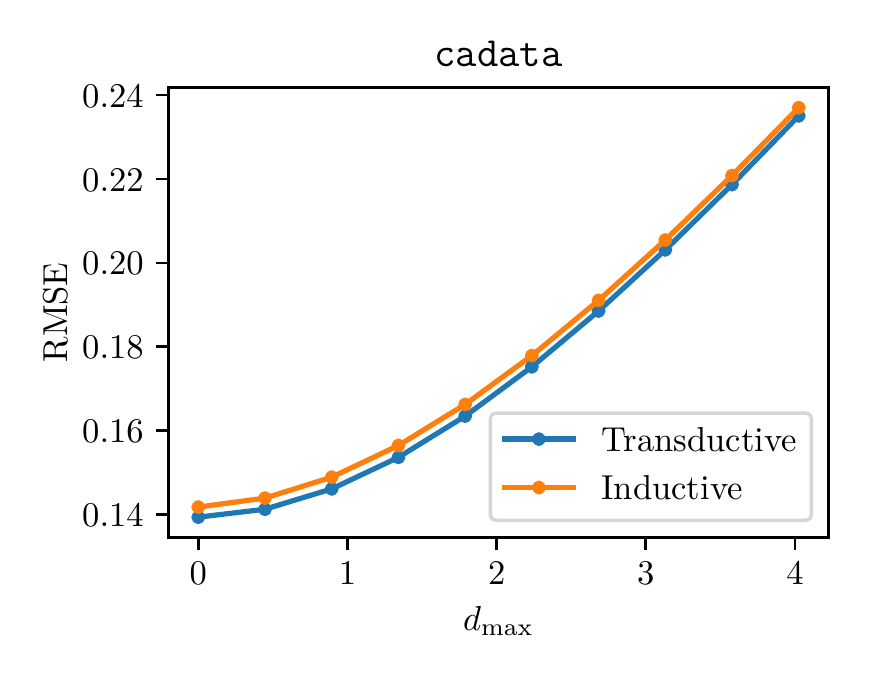}
    \caption{Experiment result of manifold regularization on \texttt{cadata}, here we set $n_l=500$, $n_u=3500$ and the rest $n_g=4000$ data are used for inductive learning.}
    \label{fig:manifold-regularization}
\end{figure}

In this experiment, we do both transductive setting, using the test set $\{\mX_u, \vy_u\}$ as in label propagation, and inductive setting, on a brand new set $\{\mX_{\mathrm{ind}}, \vy_{\mathrm{ind}}\}$ that never been accessed in training stage. We can see that for both settings the label poisoning attack algorithm has equally good performance.


\end{document}